%% file: manuscript.tex
\documentclass[letterpaper]{article}
\usepackage{aaai24} 
\usepackage{times} 
\usepackage{helvet} 
\usepackage{courier} 
\usepackage[hyphens]{url} 
\usepackage{graphicx} 
\usepackage{graphicx} 
\usepackage{natbib} 
\usepackage{caption} 
\usepackage{algorithm}
\usepackage[]{algpseudocode}
\usepackage{tabularx}
\usepackage{booktabs}

\usepackage{amsfonts}
\usepackage[table]{xcolor}
\usepackage{amsthm}
\usepackage{amsmath}

\input{newcommands}

\usepackage{newfloat}
\usepackage{listings}
\DeclareCaptionStyle{ruled}{labelfont=normalfont,labelsep=colon,strut=off} 
\floatstyle{ruled}
\newfloat{listing}{tb}{lst}{}
\floatname{listing}{Listing}
\title{$\mathcal{U}$-Trustworthy Models.\\
Reliability, Competence, and Confidence in Decision-Making}
\author {
    Ritwik Vashistha,\textsuperscript{\rm 1 \footnote{Corresponding Author}}
    Arya Farahi, \textsuperscript{\rm 1}
}
\affiliations {
    \textsuperscript{\rm 1} The Univertsity of Texas at Austin\\
    ritwik.v@utexas.edu, arya.farahi@austin.utexas.edu
}

\usepackage{bibentry}

\begin{document}

\maketitle

\begin{abstract}

    With growing concerns regarding bias and discrimination in predictive models, the AI community has increasingly focused on assessing AI system trustworthiness. Conventionally, trustworthy AI literature relies on the probabilistic framework and calibration as prerequisites for trustworthiness. In this work, we depart from this viewpoint by proposing a novel trust framework inspired by the philosophy literature on trust. We present a precise mathematical definition of trustworthiness, termed $\mathcal{U}$-trustworthiness, specifically tailored for a subset of tasks aimed at maximizing a utility function. We argue that a model's $\mathcal{U}$-trustworthiness is contingent upon its ability to maximize Bayes utility within this task subset. Our first set of results challenges the probabilistic framework by demonstrating its potential to favor less trustworthy models and introduce the risk of misleading trustworthiness assessments. Within the context of $\mathcal{U}$-trustworthiness, we prove that properly-ranked models are inherently $\mathcal{U}$-trustworthy. Furthermore, we advocate for the adoption of the AUC metric as the preferred measure of trustworthiness. By offering both theoretical guarantees and experimental validation, AUC enables robust evaluation of trustworthiness, thereby enhancing model selection and hyperparameter tuning to yield more trustworthy outcomes.
\end{abstract}

\section{Introduction}

In recent years, the AI community has expressed growing concerns about bias and discrimination embedded within predictive models. These concerns have prompted a shift in focus from performance to fairness and trustworthiness as a pilar of model evaluation \citep{mehrabi2021survey,eshete2021making}. While fairness aims to mitigate disparate impacts on different population groups, trustworthiness encompasses broader notions of model reliability, robustness, competence, generalization, explainability, transparency, reproducibility, privacy, security, and accountability \citep{serban2021practices,von2021transparency,li2023trustworthy,broderick2023toward}. In this paper, by borrowing from the philosophy literature, we develop a theory of trustworthiness from a lens of reliability and competence and investigate its implications for classification models in the context of decision-making.

Our theoretical framework is rooted in competence-based trust theories, which draws from the philosophical literature on the relation between trust, reliability, and competence \citep{baier1986trust,jones1996trust,ryan2020ai,alvarado2022kind}. Epistemologically, trust is commonly understood as the act of placing confidence in a source of information or in an agent to perform a task based on its perceived competence to be accurate and reliable. In the context of predictive models and decision-making, this notion translates into our reliance on a model that consistently demonstrates competence in achieving its goal in a decision-making task. In competence-based trust theories, trust is described through a three-part relationship ``A trusts B to do X,'' \citep{horsburgh1960ethics,ryan2020ai,von2021transparency,alvarado2022kind,afroogh2023probabilistic} where, in our case, A represents the end-user, B represents the predictive model, and X specifies the delegated task. What counts as good reasons for trusting is a matter of considerable debate, but at the very least, that B is capable or competent to do X is necessary for A to trust B. Since our interest lies in the trustworthiness of B rather than A's trust, we reframe the investigation as a two-part inquiry, ``B is trustworthy to do X.'' Formalizing this inquiry provides a foundation for trustworthiness evaluation and is the primary aim of this work. 

To achieve this goal, this work first constructs a mathematical framework for trustworthy evaluation; to evaluate the claim of whether ``B is trustworthy to do X,'' where B is a predictive model, and X is a subset of decision-making tasks. Establishing trustworthiness requires guaranteeing \emph{reliance}, \emph{competence}, and \emph{confidence}. Reliance reflects the user may rely on the model to achieve its promised goal(s), while competence provides theoretical guarantees that there exists no other model that can achieve a superior result in task X. Confidence is a statistical claim that, given the existing empirical evidence, B is competent.

Traditionally, probabilistic assessment or risk calibration has been considered a crucial aspect of model trustworthiness evaluation, as it ensures that predicted risks align with observed frequencies of outcomes and provides some information regarding the model uncertainty \citep{crowson2016assessing,pleiss2017fairness,hebert2018multicalibration,la2022proportional,afroogh2023probabilistic}. However, its role in determining the trustworthiness of a model, as described above, remains an open question. Drawing upon competence-based trust theories, we challenge this prevailing belief that risk calibration alone is sufficient or necessary to establish trustworthiness \citep{barocas2017fairness,la2022proportional,afroogh2023probabilistic}. To do so, we analyze the limitations of calibration and closely related metrics as a standalone condition for trustworthiness. This work reveals the shortcomings of the probabilistic framework \citep{afroogh2023probabilistic} in capturing the holistic notion of model trustworthiness. Additionally, our results uncover a limitation in the traditional metrics, such as accuracy, that are used for model comparison and hyper-parameter tuning. We find that relying solely on accuracy or related measures can lead to misleading conclusions regarding the trustworthiness of a model.

\section{Related Work}

Trustworthy AI has garnered widespread attention from practitioners, policy-makers, and AI developers alike, solidifying its position as a frontrunner in addressing pressing societal concerns surrounding AI bias and discrimination \citep{li2023trustworthy}. While prediction accuracy is undoubtedly an essential aspect of trustworthiness, factors such as robustness, transparency, reproducibility, replicability, stability, interpretability, and consistency are integral to trustworthiness \citep{broderick2023toward}. These aspects elucidate how an AI system performs under varying conditions, ensures unbiased predictions, and maintains consistent outputs \citep{varshney2019trustworthy,serban2021practices,von2021transparency}.

Calibration, which is closely related to the probability theory of trust \citep{afroogh2023probabilistic}, is claimed to be as one of the key requirements for trustworthy AI and used by many practitioners \citep[e.g.,][]{safavi2020evaluating,tomani2021towards}. By calibrating a classifier, it can be ensured that the predicted probabilities of the classifier more accurately reflect the true likelihood of each outcome, thereby increasing trust in the model's predictions. The literature in trustworthiness has heavily skewed on evaluating calibration of a predictive model \citep{ murphy1977reliability,naeini2015obtaining,kumar2018trainable,widmann2019calibration} and developing post-processing calibration methods for risk management \citep[e.g.,][]{murphy1977reliability,platt1999probabilistic, zadrozny2002transforming, guo2017calibration}.  

However, AI trustworthiness literature extends beyond the probabilistic paradigm. In this context, interpretability is proposed as a gateway to establishing trust \citep{ribeiro2016should}. Moreover, trustworthiness is further explored through the creation of diverse trust scores \citep{jiang2018trust, wong2020much}, as well as learning-based approaches like enhancing the loss function \citep{luo2021learning}. Our work reexamines the foundation of AI trustworthiness and proposes a novel, competence-based trustworthiness paradigm.

\section{Problem Setup}

In our study, we aim to investigate the concept of trustworthiness in the context of ``B (a predictive model) is trustworthy to do X (a decision-making task).''
This work is only concerned with a subset of tasks whose goal is to maximize a class of utility functions, hence $\mathcal{U}$-Trustworthiness. We define a $\mathcal{U}$-trustworthy model as one that possesses a decision boundary (reliability) capable of achieving maximum utility among all possible models (competence) with empirical guarantees (confidence). Next, we formalize this definition. 

Let $\bfX \in \mathcal{X}$, $Y \in \{0, 1\}$, $ \hatY \in \{0, 1\}$ denote the input features, binary outcome, and binary decision respectively; and $\mathcal{D}$ be the distribution generating $(\bfx, Y) \in \mathcal{X} \times \{0, 1\}$ pairs. $f_{\theta}: \mathcal{X} \rightarrow [0, 1]$ is a predictive model that maps inputs from $\mathcal{X}$ to a score used to assign a binary decision. $\theta$ specifies the parameters of the model that will be learned through a learning process. This work is not concerned with parameter estimation, so we assume that the model and its parameters are fixed, thus suppressing the notation $\theta$ from now on. As we will see later, the output of $f$ does not need to be interpreted probabilistically. However, after calibration, it can take probabilistic meaning. When necessary, we assume a finite test sample, denoted with $\mathcal{S}$. Finally, let  $U(\bfx, Y, \hatY): \mathcal{X} \times \{0, 1\} \times \{0, 1\} \rightarrow \mathbb{R}$ be a utility function that quantifies the desirability or usefulness of a decision outcome. Higher utility values indicate more desirable outcomes. 
See the Supplementary Materials for examples and discussion on the distinction between $Y$ and $ \hatY$.

The ultimate goal of decision-making is to assign $\hatY$ based on the output of $f(\bfx)$ and observed covariates $\bfx$ such that it maximizes the expected value of the utility function. The optimal decision rule is determined by solving 
\begin{equation} \label{eq:problem}
    g^{*}(\bfx; U, f) = \arg \max_{\widehat{g} \in \mathcal{G}} \mathbb{E}_{Y, \bfx \sim \mathcal{D}} \left[ U(\bfx, Y, \hatY) \mid  f \right].
\end{equation}
where $\widehat{Y}$ might be dependent on some decision rule $\widehat{g}$. Suppose the solution to \ref{eq:problem} has the form of
\begin{equation} \label{eq:solutoin}
    \hatY =
    \begin{cases}
            1 & \text{if} \ \ f(\bfx) \ge \widehat{g}(\bfx; U) \\
            0 & \text{otherwise}.
    \end{cases}
\end{equation}
This implies that there is a decision rule, denoted with $\widehat{g}(\bfx; U)$, that separates the prediction scores into regions associated with the binary decision assignment. We denote the maximum expected utility with $U^{(m)}$, where $U^{(m)}_{f} = \max_{\widehat{g} \in \mathcal{G}} \mathbb{E}_{Y, \bfx \sim \mathcal{D}} \left[ U(\bfx, Y, \hatY) \mid f \right]$. As we shall see, for this class of problems, we can provide generalizable trustworthiness guarantees. We acknowledge that there is no fundamental reason that the solution to Equation~\eqref{eq:problem} must take the form of Equation~\eqref{eq:solutoin}; on the contrary, we can imagine examples in which the solution to Equation~\eqref{eq:solutoin} cannot be expressed in the form of Equation~\eqref{eq:problem}.

\subsection{$\mathcal{U}$-Trustworthy}

Next, we seek to formalize the proposition ``B is trustworthy to do X'' for the class of tasks defined in Equation~\eqref{eq:problem} with the solution of the form in Equation~\eqref{eq:solutoin}. 

\begin{definition}[$\mathcal{U}$-Trustworthy]
    A model, denoted with $\widetilde{f}(.)$, is \cut\ if $U_{f}^{(m)} \le U^{(m)}_{\tildef} \ \ \forall \ U \in \mathcal{U}$ and $\forall \ f \in \mathcal{F}$. 
\end{definition} 
In simpler terms, for a class of utility functions $\mathcal{U}$, a \cut\ model is one that can be relied on to achieve the highest possible expected utility. For all $U \in \mathcal{U}$, there exists a decision boundary of a \cut\ model effectively separates the input space into regions that lead to the most desirable outcome. 

\paragraph{Reliance.} The reliance condition is met when a solution to Equation~\eqref{eq:problem} exists. This implies that model $\tildef(.)$ can be relied on to accomplish the intended goal, setting the foundation for competency, which represents the maximum achievable utility across all possible models.

\paragraph{Competency.} The competency condition is met when $U_{f}^{(m)} \leq U^{(m)}_{\tildef}$ holds for all $U \in \mathcal{U}$ and $f \in \mathcal{F}$. This ensures that no other model can attain a higher expected maximum utility than what model $\tildef(.)$ achieves.

\paragraph{Confidence.} By subjecting the claim $U_{f}^{(m)} \leq U^{(m)}_{\tildef}$ to hypothesis testing using the test data in $\mathcal{S}$, users can establish statistical confidence in the trustworthiness of the model.

\subsection{Bayes Classifier}

\begin{definition}
    Let $f^{\star}(\bfx)$ denotes the Bayes classifier, implying $f^{\star}(\bfx) = P(Y = 1 \mid \bfx)$; and $\starY$ and $\starg(\bfx; U)$ be the optimal Bayes decision and decision rule associated with the utility function $U$.
\end{definition}

\begin{proposition}
Let $f^{\star}(\bfx)$ be the Bayes classifier, then $U^{(m)}_{f} \le U^{(m)}_{\starf} \ \ \forall \ U \in \mathcal{U}$ and $\forall \ f \in \mathcal{F}$.    
\end{proposition}

This statement arises from the definition of the Bayes classifier. This proposition implies that the utility of the Bayes classifier equals that of a \cut\ classifier. Consequently, we arrive at the following theorem:

\begin{theorem} \label{theorem:main_ut}
    Model $\widetilde{f}(.)$ is \cut, if $U_{\starf}^{(m)} = U^{(m)}_{\tildef} \ \ \forall \ U \in \mathcal{U}$ and $\forall \ f \in \mathcal{F}$.
\end{theorem}

\noindent Although this theorem may seem a straightforward consequence of our definitions, it carries two significant implications. It implies that the maximum utility of a \cut\ model aligns with that of the Bayes classifier. Additionally, it streamlines the process of hypothesis testing when the null hypothesis is $U_{\starf}^{(m)} = U^{(m)}_{\tildef}$.
For the rest of this work, we adopt the notation of Table~\ref{tab:notation}.

\begin{table}[h]
  \centering
  \caption{Notation.} \label{tab:notation}
  {\scriptsize
  \begin{tabularx}{0.48\textwidth}{|X|c|c|c|c|}
    \toprule 
    Type & Classifier & Decision Rule & Assignment & Max Utility\\
    \midrule
    
    Bayes & $\starf(\bfx)$ & $\starg(\bfx)$ & $\starY$ & $U^{(m)}_{\starf}$ \\
    \midrule
    $\mathcal{U}$-trustworthy & $\tildef(\bfx)$ & $\tildeg(\bfx)$ & $\tildeY$ & $U^{(m)}_{\tildef}$  \\
    \midrule
    Generic & $f(\bfx)$ & $\hatg(\bfx)$ & $\hatY$ & $U^{(m)}_{f}$  \\
    \bottomrule 
  \end{tabularx}
  }
  \label{tab:notation}
\end{table}

\section{Limitations of Calibration Requirements in $\mathcal{U}$-Trustworthiness Assessment} 

Calibration, which is closely related to the probability theory of trust \citep{afroogh2023probabilistic}, is claimed to be as one of the key requirements for trustworthy AI and used by many practitioners \citep[e.g.,][]{safavi2020evaluating,tomani2021towards}. In this section, we revisit this claim.
\begin{definition}
    Model $f(\bfx)$ is calibrated if $P(Y = 1 \mid f(\bfx)=\alpha) = \alpha$ for all $\alpha \in [0, 1]$.
\end{definition}
We generated 400 data sets each with a sample size of 15,000 and evaluated the performance of three classifiers: (1) the Bayes classifier (blue), (2) a properly-ranked classifier (green), and (3) a calibrated adversarial classifier (red). See Supplementary Materials for the full description of the simulation. 

\paragraph{Results.} Properties of these classifiers are illustrated in Figure~\ref{fig:limitations}. The top-left present the calibration plot for these classifiers. The Bayes, by definition, and calibrated, by construction, classifiers are calibrated; while the properly-ranked classifier is miscalibrated. At first glance, this might lead to the conclusion that the calibrated and Bayes classifiers are more trustworthy for decision-making tasks, which could be reinforced by considering popular metrics like mean calibration error, Brier score \citep{brier1950verification}, and accuracy.  

\begin{figure*}[ht]
    \centering 
    \includegraphics[width=0.32\textwidth]{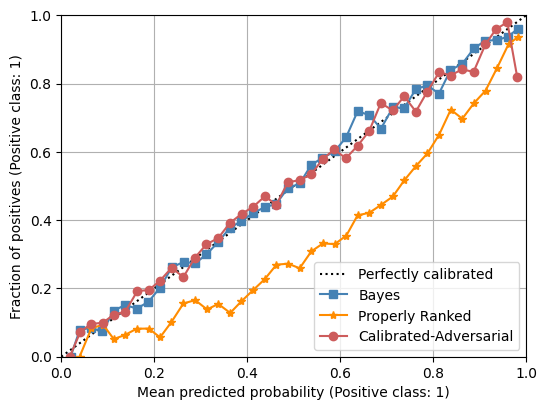}
    \includegraphics[width=0.32\textwidth]{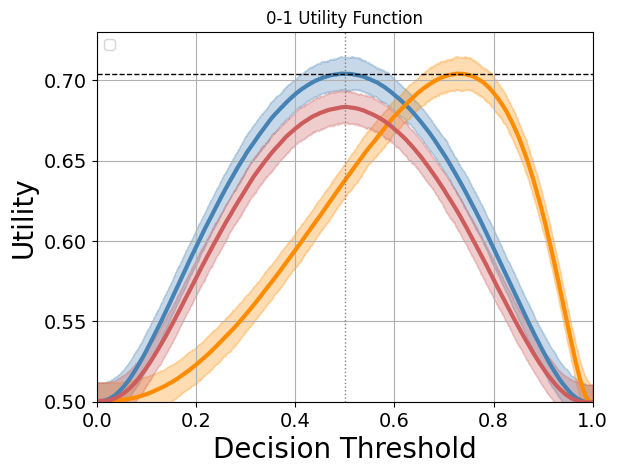}
    \includegraphics[width=0.32\textwidth]{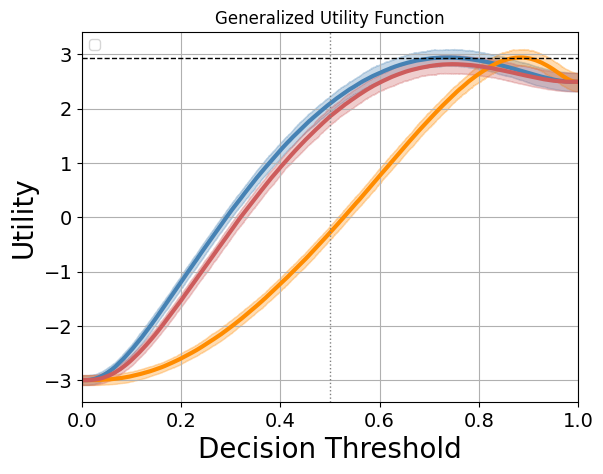} \\    
    \includegraphics[width=0.32\textwidth]{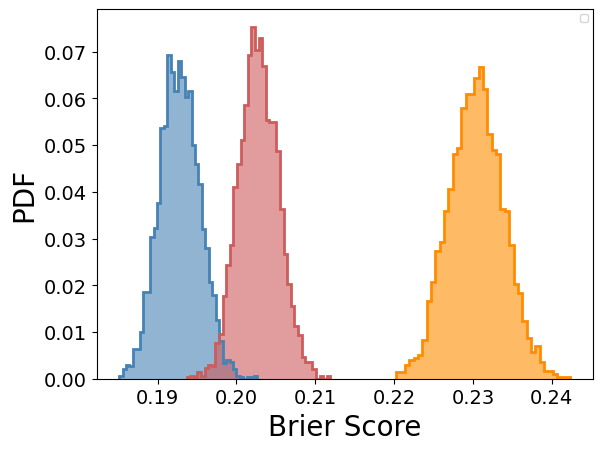}
    \includegraphics[width=0.32\textwidth]{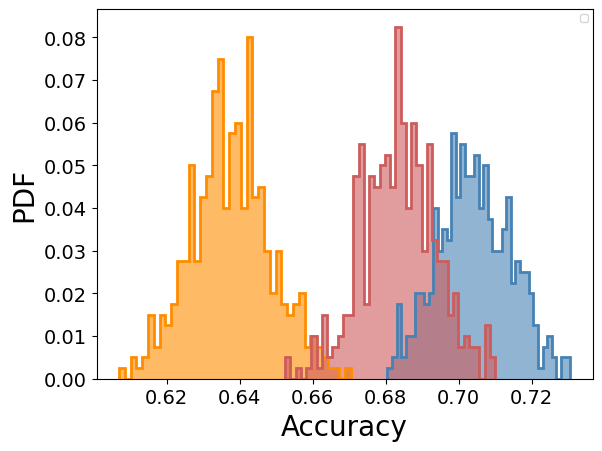}
    \includegraphics[width=0.32\textwidth]{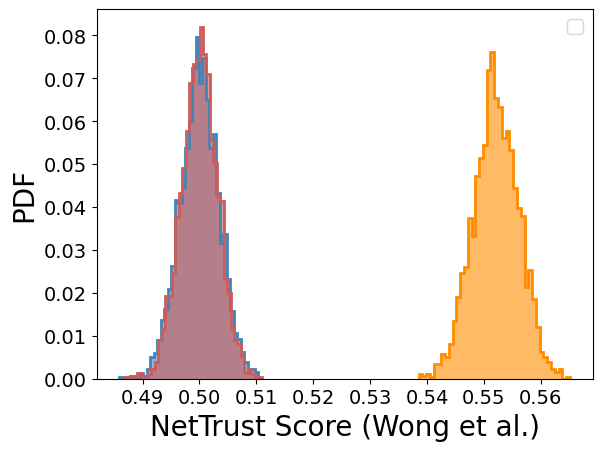} \vspace{-2mm}
    \caption{Performance of three models under different criteria. 
    \textbf{Top row}: Calibration plot (left). Utility comparison using 0-1 loss (middle) as a function of decision threshold. Utility comparison using a general utility function as a function of decision threshold (right). The confidence bands represent the 16th and 84th percentiles of 400 data realizations. \textbf{Bottom row}: Distribution of Brier score (left), Accuracy(middle), NetTrustScore (right) under 400 data realizations.}
    \label{fig:limitations}
\end{figure*}

Upon closer examination of the utility curve (top row), we observe that both the Bayes and properly-ranked classifiers have similar maximum utilities, while the calibrated classifier consistently exhibits lower maximum utility. This suggests that the properly-ranked classifier may be a \cut\ model, while the calibrated classifier falls short. This example motivates us to reevaluate the significance of calibration in trustworthy evaluation and to study the characteristics of \cut\ models more closely.

We also examine NetTrust score \citep{wong2020much}. NetTrust score ranks the Properly ranked classifier higher than calibrated and Bayes classifier, which can lead to a misleading conclusion that the Bayes classifier is suboptimal.

\paragraph{Recapitulation.} This example highlights two key findings:

\begin{itemize}
    \item The notion that a calibrated classifier is inherently trustworthy is challenged, as we demonstrate that a calibrated classifier can actually be incompetent, hence untrustworthy. Conversely, a mis-calibrated classifier can exhibit trustworthy behavior. This counterintuitive result emphasizes the need to reconsider the traditional assumption that calibration is a precursor to trustworthiness.
    \item The limitations of existing performance measures become evident in the context of trustworthy evaluation. We reveal that these measures fall short in accurately assessing the which classifier is trustworthy. This highlights the importance of reassessing evaluation metrics in capturing the nuanced aspects of trustworthiness.
\end{itemize}

\section{Characteristics of $\mathcal{U}$-trustworthy Classifiers}

\begin{definition}[Properly-Ranked Classifier]\label{bayes-equiv}
Let a classifier $f_{\rm PR}(.)$ be a properly-ranked classifier if
\[
        \forall \bfx_1, \bfx_2 \in \mathcal{X}
    \begin{cases}
            \starf(\bfx_1) > \starf(\bfx_2) \Rightarrow f_{\rm PR}(\bfx_1) > f_{\rm PR}(\bfx_2) \\
            \starf(\bfx_1) = \starf(\bfx_2) \Rightarrow f_{\rm PR}(\bfx_1) = f_{\rm PR}(\bfx_2)
    \end{cases}.
\]
\end{definition}

\begin{theorem}[$\mathcal{U}$-Competency Theorem] \label{thorem:bayes_equivalent}
Suppose for the utility class $\mathcal{U}$ with a solution of the form Equation~\eqref{eq:solutoin}. Any properly-ranked classifier is a \cut\ classifier with respect to sample $\mathcal{S}$ with $|\mathcal{S}| < \infty$.
\end{theorem}

This theorem has two important implications. First, if the solution to the general problem of Equation~\eqref{eq:problem} that is associated with the utility class $\mathcal{U}$ can be expressed with the form in Equation~\eqref{eq:solutoin}, where there is a decision boundary above which $\hatY=1$ and below which $\hatY = 0$, then any properly-ranked classifier is \cut. This provides competency criteria. A properly-ranked classifier, even with an incorrect risk estimation, is competent for the decision-making of the class described here. Theorem~\ref{thorem:bayes_equivalent} indicates that $\mathcal{U}$-trustworthiness can be achieved without calibration or proper risk estimation. Hence, for this class of problems, calibration alone is neither a necessary nor a sufficient condition for a classifier to be \cut.

While properly-ranked classifiers possess desirable properties, they do not encompass the most general class of \cut\ classifiers. The following proposition serves as an example in this regard. 
\begin{proposition} \label{prop:extending_bayes_equivalent}
Suppose a group indicator $G$ splits $\mathcal{S}$ into two non-overlapping subsets. Then, if, for each group, the properly-ranked classifier condition is satisfied, then the classifier is \cut.     
\end{proposition}
According to the proposition, if the sample $\mathcal{S}$ can be split into two non-overlapping subsets based on a group indicator $G$, and within each group separately, we get the correct ranking, then the classifier is \cut. However, it is important to note that the properly-ranked condition may not be satisfied when comparing data points across the two groups. Therefore, while the properly-ranked condition is sufficient, it is not necessary for a classifier to be \cut. In general, however, we do not consider the classifiers in Proposition~\ref{prop:extending_bayes_equivalent} as \cut\ unless the grouping is known or is learned. If such grouping exists, but it is unknown, additional search and considerations on the user's side are required to guarantee $\mathcal{U}$-trustworthiness.

\subsection{Cost-sensitive trustworthy classifiers}

Let $\mathcal{U}$ be a class cost-sensitive utility functions that are given by a weighted combination of the four fundamental population quantities closely related to elements of the ``confusion matrix'' - true positives, false positives (a.k.a. type-I error), false negatives (a.k.a. type-II error) and true negatives as defined below
\begin{equation}
    \begin{cases}
        {\bf TP} = \mathbb{I}(Y = 1, \hatY = 1  \mid \bfx), \\
        {\bf FP} = \mathbb{I}(Y = 0, \hatY = 1  \mid \bfx), \\
        {\bf FN} = \mathbb{I}(Y = 1, \hatY = 0  \mid \bfx), \\
        {\bf TN} = \mathbb{I}(Y = 0, \hatY = 0  \mid \bfx),
    \end{cases}
\end{equation}
where $\mathbb{I}$ is the indicator function.

\begin{definition}
The cost-sensitive utility family is defined as
\begin{equation}
    \mathcal{U}(\{a_{ij}\}) = a_{11}  {\bf TP} - a_{01} {\bf FP} - a_{10} {\bf FN} + a_{00} {\bf TN}. 
\end{equation}    
where $a_{ij} \geq 0$ for all $i, j \in \{0, 1\}$.
\end{definition}

\begin{example}
The 0-1 loss function belongs to the family of cost-sensitive utility functions.
\end{example}

This is straightforward to show by setting  $a_{11} = a_{00} = 1$ and $a_{10} = a_{01} = 0$. This is equivalent to 0-1 loss function. 

\begin{lemma} 
The decision rule for the Bayes classifier is 
    \begin{equation*}
        \starg = \frac{a_{01} + a_{00}}{a_{11} + a_{00} + a_{10} + a_{01}} 
    \end{equation*}    
\end{lemma}

This lemma is the reliability condition that suggests for cost-sensitive utility function that there exists a solution of form Equation~\eqref{eq:solutoin}. We also note that this class of utility functions is characterized by the coefficients $a_{ij}$, which in principle can be functions of $\bfx$, where $\bfx$ represents additional contextual information. For example, consider the utility associated with the survival of young teens, which might be higher compared to older individuals in certain scenarios. Similarly, the costs associated with the passing away of individuals could also vary based on contextual factors.

\begin{theorem}
   Let $\mathcal{U}$ be the cost-sensitive utility class. Properly-ranked classifiers are \cut. 
\end{theorem}

\subsection{Equity-aware trustworthy classifiers}

There has been a surge of interest in formulating utility functions that simultaneously account for both efficiency and equity. In this section, we 
utilize a class of equity-aware functions proposed by \citet{kleinberg2018algorithmic} and demonstrate that properly-ranked classifiers are \cut. First, we define the compatibility criteria and a class of equity-aware utility functions. 

\begin{definition}[Compatibility]
The utility function $\phi$ is compatible with the Bayes classifier if the
following natural monotonicity condition holds. If $S$ and $S^{\prime}$ are two sets of $\mathcal{X}$ of the same size, sorted in descending order of $\starf(\bfx)$, and $\starf(\bfx)$ of the $i^{\rm th}$ item in $S$ is at least as large as $\starf(\bfx)$ of the $i^{\rm th}$ item in $S^{\prime}$ for all $i$, then $\phi(S_i) \ge \phi(S^{\prime}_i)$.
\end{definition}

\begin{definition}[Equity-aware Utility Class]
Suppose there is a binary variable $G$ that splits the data into two non-overlapping subsets. The equity-aware utility family is defined as
\begin{equation}
    \mathcal{U}(\phi, \gamma) = \phi(S) + \gamma(S), 
\end{equation} 
where $S \subseteq \mathcal{X}$, $\phi(S) \in \Phi$ is compatible with Bayes probability, and $\gamma(S) \in \Gamma$ is monotonically increasing in the number of items in $S$ who have $G = 1$. 
\end{definition}

The first term characterizes the benefit while the second term characterizes the fairness. An equity-aware decision-maker seeks to maximize $U(S) = \phi(S) + \gamma(S)$. This utility function class and the following Lemma follow that of \citet{kleinberg2018algorithmic}. 

\begin{lemma}[Theorem 1, \citet{kleinberg2018algorithmic}] 
For some choice of $K_0$, from group $G=0$, and $K_1$, from group $G=0$, with $K_0 + K_1 = K$, the solution that maximizes utility in the $G = 0$ and in the $G = 1$ group are the ones with the
highest $\starf(\bfx)$.
\end{lemma}

This lemma provides the reliability condition.

\begin{theorem}
Let $\mathcal{U}$ be an equity-aware utility class. Properly-ranked classifiers are \cut.
\end{theorem}

\section{AUC and Its Relation to $\mathcal{U}$-Trustworthiness}

An ROC curve provides a graphical representation of classifier performance by comparing the true positive rate (TPR) to the false positive rate (FPR) across various decision thresholds. AUC -- the area under the ROC curve -- is a numerical measure of the classifier's performance. It quantifies the classifier's ability to rank a randomly selected positive example higher than a randomly chosen negative example, given that the positive class is ranked higher than the negative class \citep{hanley1982meaning}.

\subsubsection{Pairwise estimator of AUC.}

In probabilistic terms, the AUC represents the probability of correctly ranking the two examples \citep{hanley1982meaning}. The pairwise estimator for the AUC is commonly known as the Wilcoxon-Mann-Whitney statistic \citep{agarwal2005generalization}. It is calculated by comparing all possible pairs of observations, where one observation belongs to class 1 and the other belongs to class 0, 
\begin{equation} \label{eq:AUC_def}
{\rm AUC}(f) = \mathbb{E} \left[ \mathcal{H}(f(\bfx^{+}) - f(\bfx^{-})) \right].    
\end{equation}
$(\bfx^{+}, \bfx^{-})$ is a pair on i.i.d. draws from class 1 and zero. The expected value is computed over $\{ (\bfx^{+}, \bfx^{-}) \in \mathcal{D}^{+} \times \mathcal{D}^{-}\}$ where $\mathcal{D}^{+/-}$ is the distribution over data with class 1/0.  $\mathcal{H}(.)$ is the Heaviside step function which returns $1$ if the argument is positive, 1/2 if the argument is zero, and $0$ otherwise. An estimator of this expected value is
\begin{equation} \label{eq:AUC_estimator}
   \widehat{\rm AUC}(f) = \frac{1}{|S^{+}||S^{-}|} \sum_{i=1}^{|S^{+}|} \sum_{j=1}^{|S^{-}|} \mathcal{H}(f(\bfx_i)-f(\bfx_j))
\end{equation}
where $S^{+/-} = \{ \bfx \in S : y=1/0 \}$. The properties of this estimator are studied extensively in the literature \citep{airola2011experimental,agarwal2005generalization,cortes2004confidence}. 

\begin{theorem} [$\mathcal{U}$-Competency Measure] \label{theorem:Competency-Measure}
Let $\mathcal{U}$ be a utility class with a decision boundary of Equation~\eqref{eq:solutoin}. If and only if $f_{\rm PR}$ is a properly-ranked classifiers then ${\rm AUC}(\starf) =  {\rm AUC}(f_{\rm PR})$.
\end{theorem}

Theorem~\ref{theorem:Competency-Measure} provides a theoretical justification for why AUC may be used as a measure of competency, implying that if the AUC of a classifier is equal to the Bayes classifier, then the classifier is competent to achieve the maximume possible expected utility. Unlike the measures constructed by the confusion matrix entries, the AUC is the precision of pairwise rankings \citep{hanley1982meaning}. A properly-ranked classifier produces the same ranking as the Bayes classifier. Hence, both classifiers are expected to exhibit similar AUC performance, as implied from the above theorem, while their error rate, accuracy, and calibration can differ significantly \citep{cortes2004confidence,huang2005using}.

The literature regarding the suitability of AUC as an evaluation metric is subject to varying opinions. For instance, \citet{huang2005using} have supported the use of AUC by providing evidence that it has greater discriminative ability than accuracy, while \citet{lobo2008auc}  have cautioned against using AUC, primarily because it does not assess goodness-of-fit or might result in poor calibration. However, our results contend that in applications where utility maximization is a priority, AUC or potentially other ranking quality metrics may be more favorable. These applications include tasks where the relative ordering of actions is crucial for decision-making, such as recommendation systems \citep{schroder2011setting}, information retrieval \citep{nguyen2016probabilistic}, or task delegation \citep{farzaneh2023collaborative}. In the next section, we provide empirical evidence that utilizing AUC and accuracy can lead to different results, and when utility maximization is a priority, one should rely on AUC.

\section{Applications and Empirical Evidence}
In this section, we illustrate various practical benefits of using AUC as a performance measure in the context of $\mathcal{U}$-trustworthiness. We provide empirical evidence that AUC outperforms popular performance metrics such as accuracy in model comparison and hyperparameter tuning. We recall the definition of $\mathcal{U}$-trustworthiness, a model that achieves the highest maximum expected utility. So a model with the highest maximum expected utility is more trustworthy. When not mentioned, we use a 0-1 utility function $U(f) = {\bf TP} + {\bf TN}$. We note that the decision threshold, $\hatg$, might be different across the models. We first solve for $\hatg$ by varying the decision threshold, and then report the maximum utility. We additionally use a more complex utility function to illustrate the results hold as long as it belongs to the class of utility function discussed above. Our prediction task is to predict which household is a homeowner using the 2019 American Housing Survey data. See Supplementary Materials for the description of data sets and additional experiments.

\begin{table}[h]
\caption{Model selection results.} \vspace{-1mm}
\centering
\scriptsize{
\begin{tabularx}{0.46\textwidth}{|X|c|c|c|}
\toprule
 \textbf{Measure}  & \textbf{RF} & \textbf{LR}  & \textbf{kNN}  \\ 
\midrule
    AUC & ${\bf 0.797 \pm 0.003}$ & $0.788 \pm 0.004$ & $0.785 \pm 0.003$   \\
    Accuracy & $0.731 \pm 0.005$ & ${\bf 0.736 \pm 0.005}$ & $0.731 \pm 0.004$  \\
    Brier   & $0.182 \pm 0.001$ & ${\bf 0.177 \pm 0.002}$ & $0.180 \pm 0.002$  \\
    NetTrust & $0.580 \pm 0.001$ & $0.591 \pm 0.002$ & ${\bf  0.598 \pm 0.002}$ \\
\midrule
    $U^{(m)}$   & ${\bf 0.744 \pm 0.004}$ & $0.739 \pm 0.004$ & $0.735 \pm 0.004$  \\
  \bottomrule                          
\end{tabularx}
\label{tab:application} 
}
\end{table}

\begin{figure}[ht]
    \centering 
    \includegraphics[width=0.23\textwidth]{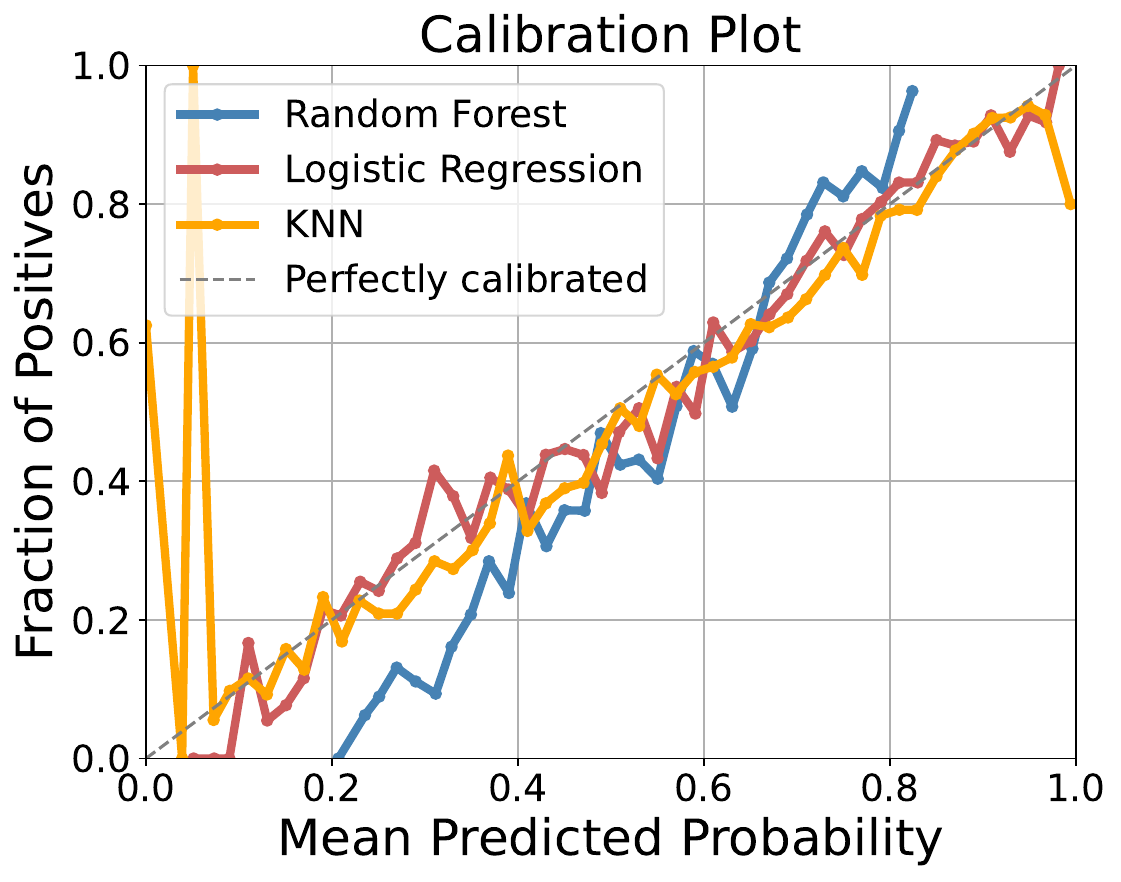} 
    \includegraphics[width=0.23\textwidth]{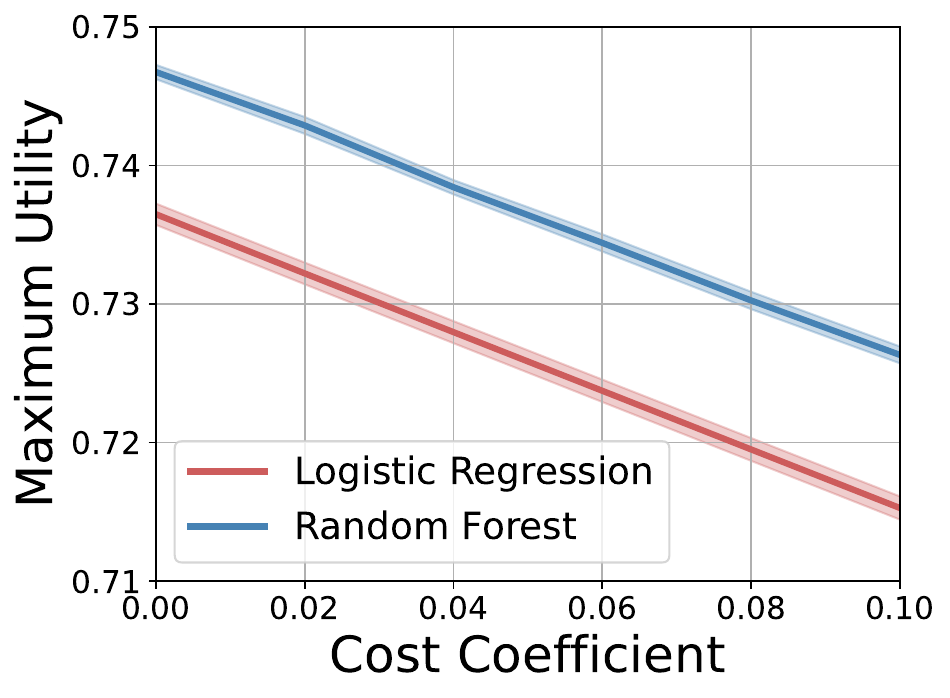} \vspace{-2mm}
    \caption{{\bf Left}: Average calibration curve. {\bf Right}: The performance of Logistic Regression, selected based on accuracy/Brier score, and Random Forest, chosen based on AUC, for a class of utility functions specified by parameter $c$, cost coefficient.
    The average maximum utility on the test sample for 20 random test/train realizations and the shaded region is 68\% error on the mean. }
    \label{fig:model_comp}
\end{figure}

\subsection{Model Selection} 
\label{model-sel}

During model selection, conflicting performance measures can lead to uncertainty about which metric should be given priority. To address this issue, we argued that AUC should be the preferred choice if the ultimate goal is $\mathcal{U}$-trustworthiness. This study compares several popular performance measures across different models using data from the homeownership dataset. Additional experiments and results can be found in the Supplementary Materials. For model comparison, we consider four metrics: AUC, Brier score, accuracy, and NetTrust score, and three models: Logistic Regression, Random Forest, and k-nearest neighbors. The evaluation results are presented in Table~\ref{tab:application}. Notably, relying solely on the accuracy or Brier score would select Logistic Regression, and NetTrust would select kNN as the preferred model. This conclusion could be further reinforced by examining the calibration properties of Logistic Regression compared to Random Forest. (the left panel of Figure~\ref{fig:model_comp}). However, when considering AUC, Random Forest emerges as the preferred choice. Given that the ultimate goal is to maximize expected utility (\cut\ setting), Random Forest should be the preferred choice, as it aligns with the AUC criteria for $\mathcal{U}$-trustworthiness. This observation highlights the importance of AUC as a reliable measure for model selection when considering the ultimate objective of maximizing expected utility. 

The second experiment evaluates the above claim but for a class of utility functions. We now explore the maximum expected utility for a class of utility functions specified by parameter $c$, defined as
\begin{equation}
\mathcal{U}[c] = {\bf TP} + {\bf TN} - c \times {\bf FP} - 0.5 c \times  {\bf FN}.
\end{equation}
In the right panel of Figure~\ref{fig:model_comp}, we present the average maximum utility on the test sample for 20 random test/train realizations. The model selected based on AUC consistently outperforms the model with lower AUC but higher Brier or accuracy score in terms of expected maximum utility. This demonstrates the significance of AUC as a reliable measure for model comparison in the context of $\mathcal{U}$-trustworthiness.

\subsection{Hyper-parameter Tuning}
Tuning of hyper-parameters is an integral part of the model-building procedure. Next, we focus on a k-nearest neighbor classifier with $k$ as the hyperparameter and investigate the impact of tuning based on different performance measures. Specifically, we compare the use of AUC and accuracy as the performance metrics during the tuning process. To conduct our analysis, we employ a homeownership dataset and perform 20-fold cross-validation to fine-tune $k$. We evaluated the model's performance using both AUC and accuracy metrics during the tuning process. The left panel of Figure \ref{fig:example_Brier_score} shows the results of tuning using AUC (blue curve) and accuracy (red curve) as performance measures. Next, we maximize the utility by varying the decision threshold. We use a non-trivial utility function that changes with the age of the householder. Let the reward of ${\bf TP} = {\bf TN}$ be 1, and the costs of the ${\bf FP}$ and ${\bf FN}$ be $C({\rm FP}) = (1 - {\rm Age} / 100) \times 3 $ and $C({\rm FN}) = (1 - {\rm Age} / 100) \times 0.5$, and the utility function for the test sample is
$U = \frac{1}{n_{\rm test}}\sum_{i=1}^{n_{\rm test}}\left[\mathbb{I}(y_i - \hat{f}_i) \times R(x_i, y_i) - |y_i - \hat{f}_i| \times C(x_i, y_i) \right]$.

Finally, we train two models, one with $k=150$ and one with $k=200$, and compute and report the utility of the test sample by varying the decision threshold (The right panel of Figure \ref{fig:example_Brier_score}). We find that selecting the hyperparameter based on AUC ($k=200$) leads to a model with higher utility than one with a hyperparameter based on accuracy ($k=150$). This observation is evident in both graphs, emphasizing the superiority of AUC as a performance measure when aiming to maximize utility. To account for randomness, we repeated the experiment with 200 random data realizations, and the line shows the mean, and the shaded region is the standard error on the mean. This result supports our claim that AUC is superior to accuracy, where utility maximization is the eventual goal, even for nontrivial utility functions.

\begin{figure}[ht]
    \centering 
    \includegraphics[width=0.24\textwidth]{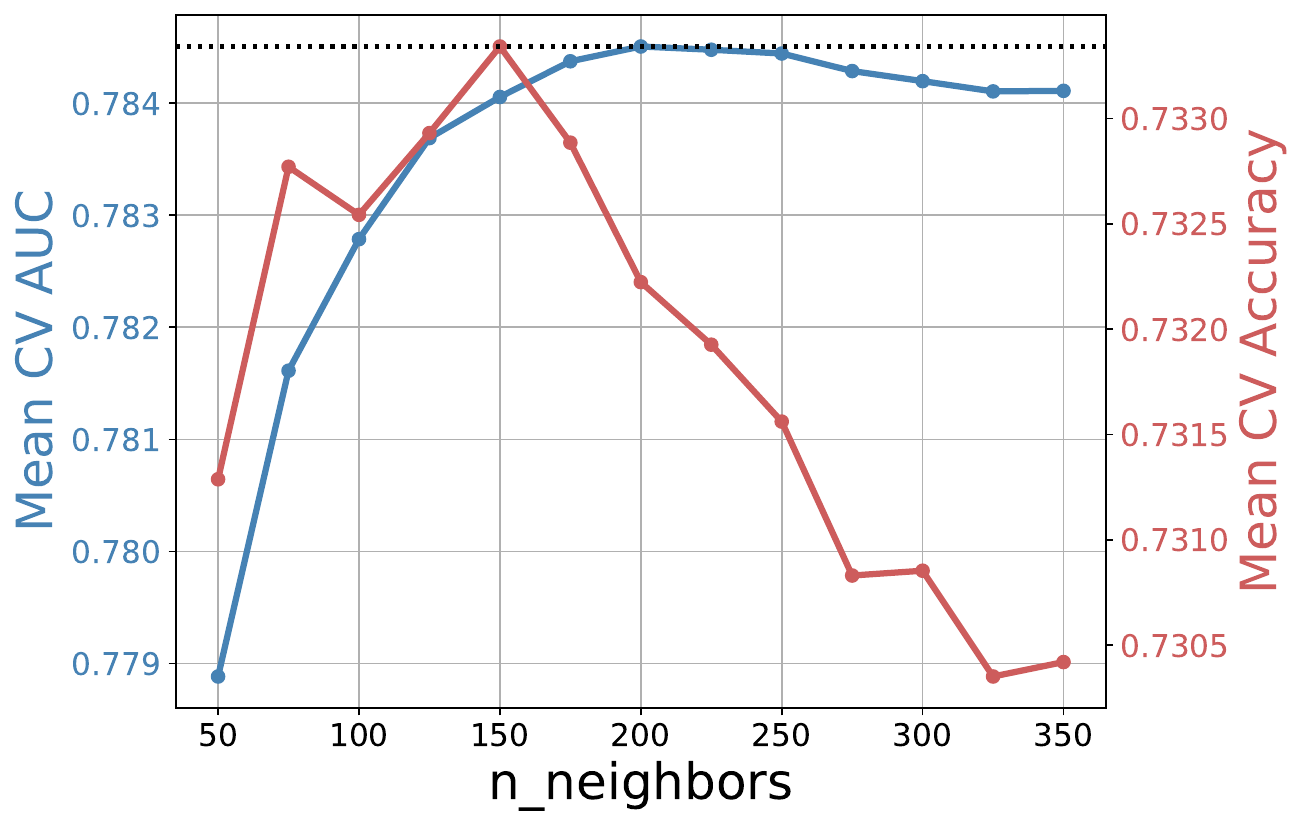} 
    \includegraphics[width=0.21\textwidth]{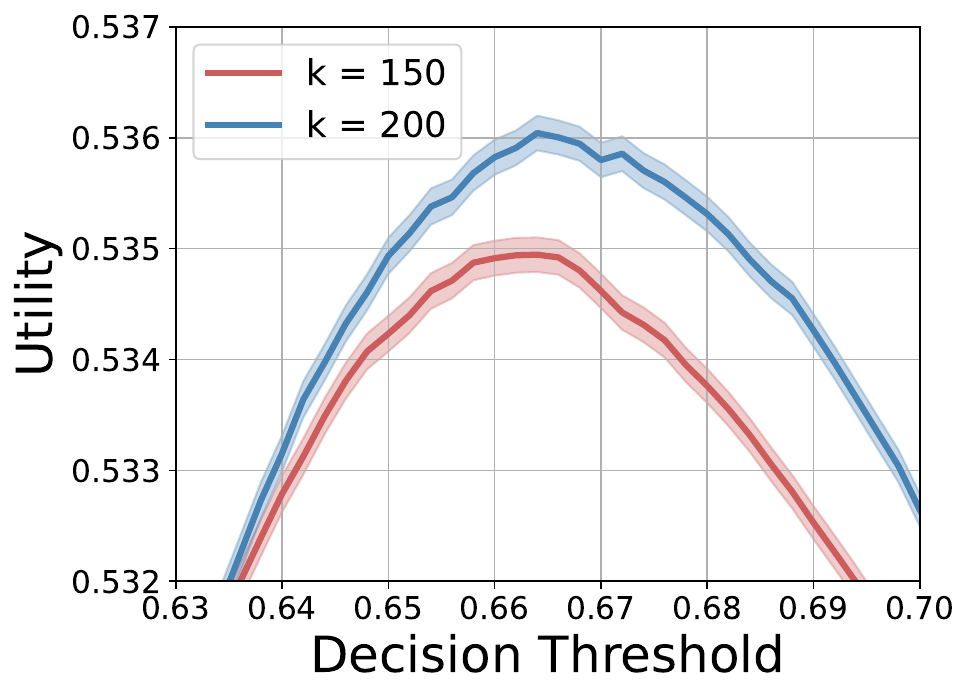} \vspace{-2mm}
    \caption{Hyper-parameter Tuning of k-NN. 
    \textbf{Left}: Average cross-validation performance vs. \texttt{n\_neighbors}. The optimal $k$ for (accuracy) AUC is (150) 200, indicated by the dotted horizontal line. \textbf{Right}: Utility as a function of decision threshold for $k = 150$ (red curve) and $k = 200$ (blue curve). The shaded region represents the standard error on mean based on 200 random test/train realizations. }
    \label{fig:example_Brier_score}
\end{figure}

\section{Discussion}

AUC has faced criticism for its limited consideration of predicted probability values and the model's goodness-of-fit \citep{lobo2008auc}. Other findings indicate that AUC surpasses metrics built on the elements of the confusion matrix, such as accuracy or the Matthew correlation coefficient, in terms of discriminative power \citep{huang2005using,halimu2019empirical}. Our results provide a clear guideline for selecting the appropriate model performance measure, specifically focusing on AUC, within the class of problems concerned with utility maximization. The findings suggest that for such problems, calibration and goodness-of-fit, although closely related measures are neither necessary nor sufficient conditions for $\mathcal{U}$-trustworthiness. It is important to note that $\mathcal{U}$-trustworthiness definition does not apply to problems requiring an unbiased risk estimate or involving objectives other than maximizing utility. While our proposed trust framework can be extended to a broader class of problems, we acknowledge its limitations in generalizing beyond utility maximization concerns.

\subsection{Limitations}

This work comes with certain limitations that should be acknowledged and provides a path for future studies. First and foremost, we recall that our framework is grounded on the competence-based trust theory; but there are other theoretical trust/trustworthiness frameworks proposed and used by the AI community \citep{toreini2020relationship, serban2021practices, li2023trustworthy}. Our framework establishes a foundation on hypothesis testing the claim ``A trusts B to do X.'' To simplify the problem further, we remove trustor A from the traditional three-part trust relation. This is the first step towards establishing full, three-part trustworthiness. Gaining users' trust requires additional domain-specific efforts in utilizing the predictive model's outcomes 
\citep{chatzimparmpas2020state,boyd2023value,mittermaier2023collaborative}. Secondly, the theorem presented in this work are for binary classifiers and do not generalize to the multi-class case. Furthermore, the \cut\ evaluation framework proposed in this work is specifically designed for a subset of tasks that seek to maximize the expected utility with solutions in the form of Equation~\eqref{eq:solutoin}. Consequently, this definition of trustworthiness should not be generalized to other tasks, such as risk mitigation or population inference, potentially rendering a model trustworthy for one set of tasks but not others. Lastly, the use of AUC as an evaluation metric in this work presents its own challenges, as the AUC of a \cut\ model remains unknown, making hypothesis testing difficult. Further research is needed to understand better the relationship between the \cut\ AUC and performing hypothesis testing, that is related to the last requirement of \cut\ which is confidence.

\subsection{Distinction with Fairness}

We differentiate between the fairness concepts and trustworthiness, as defined here, specifically in the context of the equity-aware utility class. A predictive model that assigns rankings or a risk score to individuals cannot be characterized as fair or unfair. Fairness becomes relevant when this ranking is employed to make decisions about individuals (who get the load, who will be released on bail, who will be hired). In the absence of decision-making and a decision rule, fairness does not come into play. Conversely, \cut\ refers to an inherent characteristic of a predictive model with respect to a utility function class. It asks whether the maximum achievable utility can be realized for every utility function belonging to this class or not, but which utility function should be used is a concern of fairness. 

\section{Conclusion}

In this work, grounded in the philosophy of trust, we integrate competence-based trust theory with the evaluation of predictive models. This work establishes a foundation for modeling and quantifying the trustworthiness of predictive models in decision-making contexts. The outcome of this effort promotes responsible AI development, enhances reliability and competence, and fosters user confidence, contributing to the advancement of ethical and transparent AI systems. Moreover, embracing a non-motives-based perspective aligns with the quest for objectivity and reliability in AI technologies, paving the way for ethically sound and socially beneficial applications of AI systems.


\bibliography{ref}

\newpage
\onecolumn

\appendix

\section*{\Huge Supplementary Materials}

\vspace{5mm}

\section*{\huge $\mathcal{U}$-Trustworthy Models.\\
Reliability, Competence, and Confidence in Decision-Making}

\vspace{15mm}

\addcontentsline{toc}{section}{Appendix}

\input{appendix}

\end{document}

%% file: newcommands.tex
\newcommand{\bfx}{{\mathbf{x}}}
\newcommand{\bfX}{{\mathbf{x}}}

\newcommand{\cut}{$\mathcal{U}$-trustworthy}

\newcommand{\tildeY}{{\widetilde{Y}}}
\newcommand{\tildeg}{{\widetilde{g}}}
\newcommand{\tildef}{{\widetilde{f}}}

\newcommand{\hatY}{{\widehat{Y}}}
\newcommand{\hatg}{{\widehat{g}}}
\newcommand{\hatf}{{\widehat{f}}}

\newcommand{\starY}{{Y^{\star}}}
\newcommand{\starg}{{g^{\star}}}
\newcommand{\starf}{{f^{\star}}}

\newtheorem{theorem}{Theorem}

\newtheorem{proposition}{Proposition}

\newtheorem{lemma}{Lemma}
\newtheorem{definition}{Definition}
\newtheorem{example}{Example}

\theoremstyle{definition}



%% file: appendix.tex
\section{Outcome, Decision Assignment, and Decision Outcome}

In our context, we have defined two binary variables, $Y \in \{0, 1\}$ and $\hatY \in \{0, 1\}$, representing the outcome and decision assignment, respectively. The binary outcome could signify whether a student succeeds in a program, a patient's survival after surgery, or loan repayment status. However, it is crucial to distinguish decision assignment from outcome, as they can have distinct characteristics. For instance, a decision assignment might be determining whether a student should be admitted to a program, a patient should undergo surgery, or someone should receive a loan.

To capture the desirability or usefulness of decision outcomes based on given decision assignments, we define the utility function as $U(\bfx, Y, \hatY): \mathcal{X} \times \{0, 1\} \times \{0, 1\} \rightarrow \mathbb{R}$. It's important to note that utility, in our framework, does not always equate to efficiency. Instead, it accounts for broader considerations beyond mere efficiency, making it an essential aspect of our approach. This utility function has the capacity to incorporate factors such as fairness and equity. As an example, it allows us to prioritize a particular subset of the population for decision-making, potentially resulting in higher utility for that subset. We refer the reader to Section Equity-aware trustworthy classifiers.

\section{Limitations of Calibration. A Case Study}

\paragraph{Simulation Setup.} We initially constructed a simulated data set with three sets of normally distributed random variables $X_1$, $X_2$, and $X_3$ each of size 15000. This data set was aimed at quantitatively evaluating the calibration of the models. 

The true probability p of each instance being in class 1 is then computed based on the formula $p = {\rm sigmoid}(0.5X_1 - X_2 + 0.5X_3)$. Here, sigmoid is the sigmoid activation function which is used to transform the output to a value between 0 and 1, thus allowing it to represent a probability. The true class labels are then generated from a Bernoulli distribution based on these probabilities.

Three classifiers were considered for this experiment. 
\begin{itemize}
    \item The first one represents a Bayes-optimal classifier which has perfect knowledge about the underlying distribution of the data. The probabilities from this classifier, p,  are the same as the one used to generate the true class labels. 

    \item The second classifier represents a Bayes-equivalent classifier. The probabilities from this classifier, $p_1$, are computed using the formula $p_1 = {\rm sigmoid}(0.5X_1 - X_2 + 0.5X_3 + 1.0)$. This classifier is Bayes-equivalent since it satisfies Definition \ref{bayes-equiv}.

    \item The third classifier represents a calibrated classifier. The probabilities from this classifier, $p_2$, are computed using the formula $p_2 = {\rm sigmoid}(0.5X_1 - X_2)$. This classifier is perfectly calibrated as it has full knowledge of the data generation process.
\end{itemize}

\section{Proof of Lemmas, Propositions, and Theorems.}

\setcounter{definition}{0}
\setcounter{lemma}{0}
\setcounter{theorem}{0}
\setcounter{proposition}{0}

For completeness and readability, we include all the definitions, lemmas, proofs, and theorems mentioned in the main text. 

\subsection{Definitions}

\begin{definition}[$\mathcal{U}$-Trustworthy]
    A model, denoted with $\widetilde{f}(.)$, is \cut\ if $U_{f}^{(m)} \le U^{(m)}_{\tildef} \ \ \forall \ U \in \mathcal{U}$ and $\forall \ f \in \mathcal{F}$. 
\end{definition} 

\begin{definition}
    Let $f^{\star}(\bfx)$ denotes the Bayes classifier, implying $f^{\star}(\bfx) = P(Y = 1 \mid \bfx)$; and $\starY$ and $\starg(\bfx; U)$ be the optimal Bayes decision and decision rule associated with the utility function $U$.
\end{definition}

\begin{definition}
    Model $f(\bfx)$ is calibrated if $P(Y = 1 \mid f(\bfx)=\alpha) = \alpha$ for all $\alpha \in [0, 1]$.
\end{definition}

\begin{definition}[Properly-Ranked Classifier]\label{bayes-equiv}
Let a classifier $f_{\rm PR}(.)$ be a properly-ranked classifier if
\[
        \forall \bfx_1, \bfx_2 \in \mathcal{X}
    \begin{cases}
            \starf(\bfx_1) > \starf(\bfx_2) \Rightarrow f_{\rm PR}(\bfx_1) > f_{\rm PR}(\bfx_2) \\
            \starf(\bfx_1) = \starf(\bfx_2) \Rightarrow f_{\rm PR}(\bfx_1) = f_{\rm PR}(\bfx_2)
    \end{cases}.
\]
\end{definition}

\begin{definition}
The cost-sensitive utility family is defined as
\begin{equation}
    \mathcal{U}(\{a_{ij}\}) = a_{11}  {\bf TP} - a_{01} {\bf FP} - a_{10} {\bf FN} + a_{00} {\bf TN}. 
\end{equation}    
\end{definition}

\begin{definition}[Compatibility]
The utility function $\phi$ is compatible with the Bayes classifier if the
following natural monotonicity condition holds. If $S$ and $S^{\prime}$ are two sets of $\mathcal{X}$ of the same size, sorted in descending order of Bayes probability $P(y \mid \bfx)$, and $P(y \mid \bfx)$ of the $i^{\rm th}$ item in $S$ is at least as large as $P(y \mid \bfx)$ of the $i^{\rm th}$ item in $S^{\prime}$ for all $i$, then $\phi(S) \ge \phi(S^{\prime})$.
\end{definition}

\begin{definition}[Equity-aware Utility Class]
Suppose there is a binary variable $G$ that splits the data into two non-overlapping subsets. The equity-aware utility family is defined as
\begin{equation}
    \mathcal{U}(\phi, \gamma) = \phi(S) + \gamma(S), 
\end{equation} 
where $S \subseteq \mathcal{X}$, $\phi(S)$ is compatible with Bayes probability, and $\gamma(S)$ and is monotonically increasing in the number of items in $S$ who
have $G = 1$.
\end{definition}

The first term characterizes the benefit while the second term characterizes the fairness. An equity-aware decision-maker seeks to maximize $U(S) = \phi(S) + \gamma(S)$.

\subsection{Problem Setup}

\begin{proposition} \label{proposition:bayes-classifier}
Let $f^{\star}(\bfx)$ be the Bayes classifier, then $U^{(m)}_{f} \le U^{(m)}_{\starf} \ \ \forall \ U \in \mathcal{U}$ and $\forall \ f \in \mathcal{F}$.
\end{proposition}

This statement arises from the definition of the Bayes classifier. This proposition implies that the utility of the Bayes classifier equals that of a \cut\ classifier. Consequently, we arrive at the following theorem.

\begin{theorem} \label{theorem:main_ut}
    Model $\widetilde{f}(.)$ is \cut, if $U_{\starf}^{(m)} = U^{(m)}_{\tildef} \ \ \forall \ U \in \mathcal{U}$ and $\forall \ f \in \mathcal{F}$.
\end{theorem}

\begin{proof}
    Let $f^{\star}(\bfx)$ be a Bayes classifier. Following Proposition~\ref{proposition:bayes-classifier}, for all $U \in \mathcal{U}$, we have
    $$
    \begin{aligned}
        U^{(m)}_{f} &\le U^{(m)}_{\starf} \ \ \forall \ f \in \mathcal{F} \\
         U^{(m)}_{f} &\le U^{(m)}_{\starf} = U^{(m)}_{\tildef} \ \ \forall \ f \in \mathcal{F} \\
         \Rightarrow U^{(m)}_{f} &\le U^{(m)}_{\tildef} \ \ \forall \ f \in \mathcal{F}
    \end{aligned}
    $$
Thus, $\widetilde{f}(.)$ is \cut.    
\end{proof}

\subsection{Characteristics of \cut\ Classifiers}

\begin{theorem}[$\mathcal{U}$-Competency Theorem] \label{thorem:bayes_equivalent}
Suppose for the utility class $\mathcal{U}$ with a solution of the form Equation~\eqref{eq:solutoin}. Any properly-ranked classifier is a \cut\ classifier with respect to sample $\mathcal{S}$ with $|\mathcal{S}| < \infty$.
\end{theorem}

\begin{proof}
If we show  that the decision assignments under $\starf(\bfx_1)$ and $f_{\rm PR}(\bfx_1)$ is the same, $\starY = \hatY \ \ \forall \bfx \in \mathcal{S}$, then $U_{\starf}^{(m)} = U_{\hatf}^{(m)}$. Then, the theorem is proved following Theorem~\ref{theorem:main_ut}. 

We partition $\mathcal{S}$ into $N$ disjoint subsets, $\mathcal{S} = \cup_{i=1}^{N} S_i$, such that for each subset $S_i = \{\bfx \in \mathcal{X} : \hatg(\bfx) = c_i \}$ where $c_i \in [0, 1]$ is a constant. 

Let $S_i^{+} \subseteq S_i$ be all the data points above the decision threshold, $S_i^{+} = \{ \bfx \in S_i : \starY=1 \}$. $S_i^{+}$ can be an empty set or a non-empty set. If it is an empty set, then we set $\hatY = 0$, and hence $\starY = \hatY = 0 \ \ \forall \bfx \in S_i$.

If $S_i^{+}$ is a non-empty set, then there exists $\bfx^{+}$ such that $\bfx^{+} = \arg \min_{\bfx \in S_i^{+}} |c_i - \starf(\bfx)|$\footnote{This property is due to the finite nature of set $S_i$, for infinite sets the minimum might not exist, and we require additional convergence conditions.}. This entails
\begin{equation*}
    \begin{cases}
        \starf(\bfx^{+}) \le \starf(\bfx) \ \ \forall \ \bfx \in S_i^{+} \\
        \starf(\bfx^{+}) > \starf(\bfx) \ \ \forall \ \bfx \in  S_i \backslash S_i^{+}
    \end{cases}
\end{equation*}
Now, we set $f_{\rm PR}(\bfx^{+}) = \hatg$. Following the definition of Properly ranked, we have 
\begin{equation*}
    \begin{cases}
        \starf(\bfx^{+}) \le \starf(\bfx) \Rightarrow f_{\rm PR}(\bfx^{+}) \le f_{\rm PR}(\bfx) \Rightarrow \hatY = 1 \ \ \forall \ \bfx \in S_i^{+} \\
        f_{\rm PR}(\bfx^{+}) > f_{\rm PR}(\bfx) 
        \Rightarrow
        \starf(\bfx^{+}) > \starf(\bfx) \ \Rightarrow \hatY = 0 \ \ \forall \ \bfx \in  S_i \backslash S_i^{+}
    \end{cases}
\end{equation*}
This implies $\hatY = \starY \ \forall \ \bfx \in S_i$. Because the above argument holds for all subsets $S_1$, $\ldots$, $S_N$, then $\hatY = \starY$  for all subsets $S_1$, $\ldots$, $S_N$ and the theorem is proved. 
\end{proof}

While properly-ranked classifiers possess desirable properties, they do not encompass the most general class of \cut\ classifiers. The following proposition serves as an example in this regard. 
\begin{proposition} \label{prop:extending_bayes_equivalent}
Suppose a group indicator $G$ splits $\mathcal{S}$ into two non-overlapping subsets. Then, if, for each group, the properly-ranked classifier condition is satisfied, then the classifier is \cut.     
\end{proposition}

\begin{proof}
   Suppose that a group indicator G splits $\mathcal{S}$ into two non-overlapping subsets, $\mathcal{S} = \mathcal{S}_1 \cup \mathcal{S}_2 $. For both subsets, the properly-ranked condition in Definition \ref{bayes-equiv} is satisfied. Then the theorem is proved following the steps of Theorem \ref{theorem:Competency-Measure}. It can be shown for each subset $\mathcal{S}_1$ and $\mathcal{S}_2$, that $\starY = \hatY \ \ \forall \bfx \in \mathcal{S}_1$ and $\starY = \hatY \ \ \forall \bfx \in \mathcal{S}_2$ and thus we would have $U_{\starf}^{(m)} = U_{\hatf}^{(m)}$.
\end{proof}

\begin{lemma} \label{lemma:cost_sens-decision_rule}
The decision rule of the cost-sensitive utility family for the Bayes classifier is 
    \begin{equation*}
        \starg(\bfx) = \frac{a_{01} + a_{00}}{a_{11} + a_{00} + a_{10}+ a_{01}}  
    \end{equation*}    
\end{lemma}

\begin{proof}
 The optimal decision rule is determined by solving 
\begin{equation}
    \starg(\bfx; U, \starf) = \arg \max_{g \in \mathcal{G}} \mathbb{E}_{Y, \bfx \sim \mathcal{D}} \left[ U(\bfx, Y, \starY) \mid  \starf \right].
\end{equation}

We can write the cost sensitive utility function as:
$$
\begin{aligned}
U(\bfx, Y, \starY) &= a_{11} \quad \text{if } Y = \starY = 1 \\
          &= a_{00}  \quad \text{if } Y =  \starY = 0\\ 
          &= - a_{10} \quad \text{if } Y = 1 ,   \starY = 0 \\
          &=  -a_{01} \quad \text{if } Y = 0 ,  \starY = 1 \\
\end{aligned}
$$

The expected utility under the cost sensitive utility function for $\starf(\bfx)$ is given as follows: 

$$
\begin{aligned}
U(\bfx) &= \mathbb{E}_{Y, \bfx \sim \mathcal{D}} \left[ U(\bfx, Y, \starY) \mid  \starf \right] \\
     &= \mathbb{E}_{\bfx} \text{ } \mathbb{E}_{Y|\bfx}\left(U(\bfx,Y,\starY)\right) \\
     &=\mathbb{E}_{\bfx} \left[ - a_{10}P(Y = 1, \starY = 0 |\bfx)  - a_{01}P(Y = 0, \starY = 1 |\bfx) + a_{00}P(Y = 0, \starY = 0 |\bfx) + a_{11}P(Y = 1, \starY = 1 |\bfx) \right] \\
     &= \mathbb{E}_{\bfx}  [-a_{10}P(Y = 1 |\bfx)P(\starY = 0 | \bfx) - a_{01}P(Y = 0|\bfx)P(\starY = 1 | \bfx) + \\ 
     & \quad \quad \quad a_{00}P(Y = 0 |\bfx)P(\starY = 0 |\bfx) + a_{11}P(Y = 1|\bfx)P(\starY = 1 | \bfx) ] \\
     &= \mathbb{E}_{\bfx} [\text{ }\left[-a_{10}P(Y = 1 |\bfx) + a_{00}P(Y = 0 |\bfx)\right]P(\starY = 0 | \bfx) \text{ } + \text{ }  \\ 
     & \quad \quad \quad \left[-a_{01}P(Y = 0 |\bfx) + a_{11}P(Y = 1 |\bfx)\right]P(\starY= 1 | \bfx)  ] \\
     &=\mathbb{E}_{\bfx} [\text{ }\left[a_{00} - (a_{10}+a_{00})P(Y = 1 |\bfx) \right]P(\starY = 0 | \bfx) \text{ } + \text{ }\\
     & \quad \quad \quad \left[ (a_{11} + a_{01})P(Y = 1 |\bfx) - a_{01} \right] P(\starY= 1 | \bfx)  ] \\
\end{aligned}
$$
Note that $U(\bfx)$ is maximized $\forall \bfx \in \mathcal{X}$ if we consider the following decision rule
$$
\begin{aligned}
\starY = 
\begin{cases}
    & 1 \quad \text{if } \ \  (a_{11} + a_{01})P(Y = 1 |\bfx) - a_{01} \ge a_{00} - (a_{10}+a_{00})P(Y = 1 |\bfx)  \\
    & 0 \quad otherwise   
\end{cases}
\end{aligned}
$$
Since $\starf(\bfx) = P(Y = 1 |\bfx)$, we can write the above as: 

\begin{equation} \label{eq: opt_rule}
\starY = 
\begin{cases}
    & 1 \quad \text{if } \starf(\bfx) \ge \frac{a_{01} + a_{00}}{(a_{01} + a_{01}) + (a_{00} + a_{11})} \\
    & 0 \quad otherwise   
\end{cases}
\end{equation}

Thus the decision rule for the Bayes classifier is 
    \begin{equation*}
        \starg = \frac{a_{01} + a_{00}}{(a_{01} + a_{01}) + (a_{00} + a_{11})} 
    \end{equation*} 

\end{proof}

\begin{theorem}
   Let $\mathcal{U}$ be the cost-sensitive utility class. Properly-ranked classifiers are \cut. 
\end{theorem}

\begin{proof}
    Because for the class of the cost-sensitive utility functions the optimal decision rule  in Equation~\eqref{eq: opt_rule} has the form of Equation~\eqref{eq:solutoin} (Lemma~\ref{lemma:cost_sens-decision_rule}),  following Theorem~\ref{theorem:Competency-Measure} properly-ranked classifiers are \cut.
\end{proof}

\begin{lemma}[Theorem 1, \citet{kleinberg2018algorithmic}] \label{lemma:equity-aware-lemma}
For some choice of $K_0$, from group $G=0$, and $K_1$, from group $G=0$, with $K_0 + K_1 = K$, the solution that maximizes utility in the $G = 0$ and in the $G = 1$ group are the ones with the
highest $\starf(\bfx)$.
\end{lemma}

\begin{proof}
    We recall $\starf(\bfx) \equiv P(y=1 \mid \bfx)$. Let $S^{*}$ be any set of size $K$ that maximizes $\phi(S) + \gamma(S)$. We can partition $S^{*} = S_0^{*} \cup S_1^{*}$ such that the $S_0^{*}$ represents those in the $G = 0$ group and $S_1^{*}$ represents those in the $G = 1$ group. Let $K_0 = |S_0^{*}|$ and $K_1 = S_0^{*}$. Let $S_0^{+}$ be the sampe of size $K_0$ in the $G = 0$ group with the highest $P(y \mid \bfx)$, and let $S_1^{+}$ be the sample of size $K_1$. Denote $S^{+}=S_0^{+} \cup S_1^{+}$. If we sort $S^{+}$ and $S^*$ in descending order of predicted performance, then the $i^{\rm th}$ sample in $S^{+}$ has predicted performance at least as large as the $i^{\rm th}$ sample in $S^*$. Hence by the compatibility of $\phi$ and $P(y \mid \bfx)$, we have $\phi\left(S^{+}\right) \geq \phi\left(S^*\right)$. $S^{+}$ and $S^*$ have the same number of members with $G=1$ by construction. This implies $\gamma\left(S^{+}\right) = \gamma\left(S^*\right)$. Because we have  $\gamma\left(S^{+}\right) = \gamma\left(S^*\right)$ and $\phi\left(S^{+}\right) \geq \phi\left(S^*\right)$, then $\phi\left(S^{+}\right)+\gamma\left(S^{+}\right) \geq \phi\left(S^*\right)+\gamma\left(S^*\right)$. Consequently, $S^{+}$ is the set that maximizes the equity-aware utility function and satisfies the conditions of the theorem.
\end{proof}

\begin{theorem}
Let $\mathcal{U}$ be an equity-aware utility class. Properly-ranked classifiers are \cut.
\end{theorem}

\begin{proof}
    Let $\mathcal{U}$ be an equity-aware utility class and that there exists a binary variable $G$ that splits the data into two non-overlapping subsets. Following Lemma~\ref{lemma:equity-aware-lemma} there are two non-overlapping sets $S_0^{*} \cup S_1^{*}$ that for each subset there is threshold above which $\hatY = 1$ and below $\hatY = 0$. Therefore, the solution to problem Equation~\eqref{eq: opt_rule} has the form of Equation~\eqref{eq:solutoin}.  Then following Proposition \ref{prop:extending_bayes_equivalent}, properly-ranked classifiers are \cut.
\end{proof}

\subsection{Performance Measure and their Relation to $\mathcal{U}$-Trustworthiness}

\begin{theorem} [$\mathcal{U}$-Competency Measure]
Let $\mathcal{U}$ be a utility class with a decision boundary of Equation~\eqref{eq:solutoin}. If and only if $f_{\rm PR}$ is a properly-ranked classifiers then ${\rm AUC}(\starf) =  {\rm AUC}(f_{\rm PR})$.
\end{theorem}
\begin{proof}
Let $C(\bfx_1)$ and $C(\bfx_2)$ denote the class labels of observations $\bfx_1$ and $\bfx_2$, respectively. Then we have that,
$$
\begin{aligned}
    {\rm AUC}(\starf) \\ 
    & \quad \quad \text{using Equation } \eqref{eq:AUC_def} \text{ we have that, } \\
    & = \mathbb{E} \left[ \mathcal{H}(\starf(\bfx^{+}) - \starf(\bfx^{-})) \mid C(\bfx^{-}) = 0, C(\bfx^{+}) = 1) \right] 
    \\   & \quad \quad \text{using Definition } \ref{bayes-equiv} \text{ we have that, } \\
                  &= \mathbb{E} \left[ \mathcal{H}(f_{\rm PR}(\bfx^{+}) - f_{\rm PR}(\bfx^{-})) \mid C(\bfx^{-}) = 0, C(\bfx^{+}) = 1) \right] \\
                  &= {\rm AUC}(f_{\rm PR})
\end{aligned}
$$
Now, we prove the inverse. We have ${\rm AUC}(f) \le {\rm AUC}(\starf) \ \ \forall f \in \mathcal{F}$. Consequently, ${\rm AUC}(\starf) = {\rm AUC}(f_{\rm PR})$ implies that $\mathcal{H}[\starf(\bfx^{+}) - \starf(\bfx^{-})] = \mathcal{H}[f_{\rm PR}(\bfx^{+}) - f_{\rm PR}(\bfx^{-})]$ for all $(\bfx^{+}, \bfx^{-}) \in \mathcal{D}^{+} \times \mathcal{D}^{-}$. $\mathcal{H}[\starf(\bfx^{+}) - \starf(\bfx^{-})] = \mathcal{H}[f_{\rm PR}(\bfx^{+}) -  \starf(\bfx^{-})]$ implies
\[
        \forall \bfx^{+}, \bfx^{-} \in \mathcal{D}^{+} \times \mathcal{D}^{-}
    \begin{cases}
            \starf(\bfx^{+}) > \starf(\bfx^{-}) \Rightarrow f_{\rm PR}(\bfx^{+}) > f_{\rm PR}(\bfx^{-}) \\
            \starf(\bfx^{+}) = \starf(\bfx^{-}) \Rightarrow f_{\rm PR}(\bfx^{+}) = f_{\rm PR}(\bfx^{-}) \\ 
            \starf(\bfx^{+}) < \starf(\bfx^{-}) \Rightarrow f_{\rm PR}(\bfx^{+}) < f_{\rm PR}(\bfx^{-}) \\ 
    \end{cases}.
\]
This is equivalent to the condition for a properly-ranked  classifier; therefore, $f_{\rm PR}$ is a properly ranked classifier. 

There are various other ranking metrics in the literature such as partial AUC \cite{yang2022optimizing, yang2022auc},  top-k measures \cite{pmlr-v119-yang20f, wang2022optimizing} and AUPRC \cite{wen2022exploring}. Any metric that can preserve ranking can be used for trustworthiness evaluation. Thus, results proved in Theorem \ref{theorem:Competency-Measure} can be extended to  ranking metrics such as partial AUC, top-k measure and AUPRC.  

\end{proof}

\section{Data Sets}

\subsection{Homeownership}

The homeownership data set used in this study is sourced from the 2019 American Housing Survey (AHS). The AHS is one of the most comprehensive national housing surveys in the U.S. and is sponsored by the Department of Housing and Urban Development while being conducted by the Census Bureau. The dataset includes various household attributes such as income, marital status, education attainment, and more. Additionally, division, binary metropolitan area, binary race, and sex variables are considered in the analysis. Further details on these variables can be found in Table~\ref{tab:vars}. For this work, the focus is on households with incomes between \$10,000 and \$500,000. After filtering out non-respondents and households with income outside this range, the final dataset comprises 48,660 households. To facilitate the modeling process, log-income transformation has been performed. No other transformations have been applied to the data.

\begin{table*}[]
\centering
\caption{Independent and dependent variables.} \label{tab:vars}
\begin{tabular}{|l|l|l|l|}
\hline
{\bf Attribute Name} & {\bf Definition}  &  {\bf Type } & {\bf Values}  \\ \hline 
BLACK          & One if at least one member of the     & Binary  & \{0, 1\}        \\
              &  household is African-American, 0 otherwise &      &         \\\hline
HHMAR          & Marital status                                                         & Categorical & -- \\ \hline
HINCP          & Household income                                                       & Numeric     & {[}\$10k, \$500k{]}  \\ \hline
HHGRAD         & Educational attainment of householder                                  & Categorical     & -- \\ \hline
HHAGE          & Age of householder                                                      & Numeric     & {[}15, 85{]}    \\ \hline
DIVISION       & Division                                                               & Categorical & -- \\ \hline
HHCITSHP       & Citizenship of householder                                              & Categorical & -- \\ \hline
NUMPEOPLE      & Count of households                                                    & Numeric     & {[}1, 18{]}     \\ \hline
HHSEX          & Sex of householder (Male=1, Female=2)                                       & Binary      & \{1, 2\}     \\ \hline
METRO          & One if in a metropolitan area, 0 otherwise.                            & Binary      & \{0, 1\}         \\\hline 
OWNER          & One if owned or being bought by someone  & Binary      & \{0, 1\} \\       
          &  in the household, 0 otherwise. &      &  \\\hline
\end{tabular}
\end{table*}

\subsection{Adult}

The adult dataset (\citet{misc_adult_2}) used in this study is taken from the UCI Machine Learning repository. The dataset utilizes data derived from the 1994 Census Bureau database to examine the predictive ability of demographic and employment-related information in determining if an individual's annual income exceeds \$50,000. The dataset consists of a total of 48,842 observations and includes a combination of 14 attributes, which encompass both continuous and categorical variables. These qualities encompass a range of dimensions within an individual's profile, including age, employment class, educational attainment, marital status, occupation, relationship status, race, gender, capital gain, capital loss, weekly working hours, and place of origin. The dataset is widely employed within the machine learning field because of its diverse range of attribute types and the binary classification problem it presents.

\subsection{Bankruptcy}

The bankruptcy dataset (\citet{misc_taiwanese_bankruptcy_prediction_572}) used in this study is taken from the UCI Machine Learning repository. This dataset stems from the real financial statements of Taiwanese companies from 1999 to 2009. It comprises multiple features that encapsulate various financial metrics, such as the ratio of net income to total assets, persistent earnings to total assets, and cash flow rate. Each instance represents a company's financial profile and is labeled to indicate if the company went bankrupt within the following year. The dataset possesses significant value in terms of its origin and composition, as it offers a unique viewpoint on the financial well-being of organizations.

\subsection{Breast Cancer}

The breast cancer dataset (\citet{misc_breast_cancer_wisconsin_(diagnostic)_17}) used in this study is taken from the UCI Machine Learning repository. Comprising 569 instances, the dataset captures features computed from a digitized image of a fine needle aspirate (FNA) of a breast mass. Each instance in the dataset consists of 30 real-valued input features, which include aspects like texture, radius, perimeter, area, smoothness, and concavity, among others. These features are utilized to predict the malignancy of the cell formations, with the diagnosis labeled either as benign or malignant. Over the years, the dataset has played an instrumental role in the development and validation of a wide range of classification algorithms in the domain of medical diagnosis.

\section{Additional Model Comparison Experiments}

Additional model selection experiments were conducted to compare various performance indicators across different models and datasets. The findings are displayed in table \ref{tab:application_2}. When considering the adult dataset, all of the measures indicate that Logistic Regression is the preferred model, and it also attains the highest level of maximum utility. However, in the case of the bankruptcy and breast cancer datasets, the NetTrust score disagrees with the other performance measures and chooses an alternative model. Given that the objective is to maximize the expected utility, it is preferable to select Random Forest as the model in both scenarios. This is because Random Forest demonstrates the maximum AUC, hence aligning with the AUC criterion for U-trustworthiness. These results align with our previous findings for the homewownership data and show that AUC is a reliable measure for model selection when considering the ultimate objective of maximizing expected utility.

\begin{table}[h]
\caption{Model selection results  }
\centering
\begin{tabularx}{0.6\textwidth}{|X|c|c|c|}
\toprule
 \textbf{Measure}  & \textbf{RF} & \textbf{LR}  & \textbf{kNN}  \\ 
\midrule
\textbf{Adult}\\
AUC & $0.890 \pm 0.003$ & ${\bf0.906 \pm 0.003}$ & $0.879 \pm 0.004$ \\
Accuracy  & $0.791 \pm 0.005$ & ${\bf0.853 \pm 0.003}$ & $0.834 \pm 0.003$ \\
Brier & $0.130 \pm 0.002$ & ${\bf0.102 \pm 0.002}$ & $0.116 \pm 0.002$  \\
NetTrust  & $0.703 \pm 0.002$ & ${\bf0.735 \pm 0.003}$ & $0.720 \pm 0.002$\\
\midrule 
$U^{(m)}$  & $0.846 \pm 0.004$ & ${\bf0.853 \pm 0.003}$ & $0.836 \pm 0.003$ \\
  \bottomrule
\textbf{Bankruptcy}\\  
AUC & ${\bf0.928 \pm 0.016}$ & $0.576 \pm 0.039$ & $0.702 \pm 0.034$ \\
Accuracy  & ${\bf0.969 \pm 0.004}$ & $0.961 \pm 0.004$ & $0.968 \pm 0.004$ \\
Brier & ${\bf0.025 \pm 0.003}$ & $0.042 \pm 0.004$ & $0.031 \pm 0.004$  \\
NetTrust  & $0.051 \pm 0.003$ & ${\bf0.090 \pm 0.004}$ & $0.057 \pm 0.003$\\
\midrule 
$U^{(m)}$  & ${\bf0.971 \pm 0.004}$ & $0.968 \pm 0.004$ & $0.968 \pm 0.004$ \\
  \bottomrule
\textbf{Breast Cancer}\\  
AUC & ${\bf 0.989 \pm 0.009}$ & ${\bf0.990 \pm 0.007}$ & $0.974 \pm 0.013$ \\
Accuracy & ${\bf0.955 \pm 0.018}$ & ${\bf0.946 \pm 0.017}$ & $0.903 \pm 0.025$ \\
Brier & ${\bf0.036 \pm 0.010}$ & ${\bf 0.039 \pm 0.012}$ & $0.078 \pm 0.014$  \\
NetTrust  & $0.618 \pm 0.040$ & $0.618 \pm 0.040$ & ${\bf0.665 \pm 0.028}$\\
\midrule 
$U^{(m)}$   & ${\bf0.966 \pm 0.016}$ & $0.961 \pm 0.015$ & $0.935 \pm 0.022$ \\
\bottomrule
\end{tabularx}
\label{tab:application_2} 
\end{table}